%% file: main.tex
\documentclass[11pt]{article}

\usepackage[letterpaper,left=1in,right=1in,top=1in,bottom=1in]{geometry}
\usepackage[T1]{fontenc}
\usepackage{mathptmx}
\usepackage{courier}
\usepackage{amsmath, amssymb, amsthm}
\usepackage{booktabs}
\usepackage{colortbl}
\usepackage{enumitem}
\usepackage{graphicx}
\usepackage{float}
\usepackage[font=small,labelfont=bf]{caption}
\usepackage{microtype}
\usepackage{titlesec}
\usepackage{tikz}
\usepackage{varwidth}
\usetikzlibrary{arrows.meta, positioning, fit, backgrounds}
\usepackage[numbers,sort&compress]{natbib}
\usepackage[colorlinks=true,linkcolor=black,citecolor=black,urlcolor=black]{hyperref}
\usepackage[most]{tcolorbox}

\graphicspath{{figures/}}

\setlist[enumerate]{topsep=0.25em,itemsep=0.2em,leftmargin=1.6em}
\setlist[itemize]{topsep=0.2em,itemsep=0.15em,leftmargin=1.4em}
\titlespacing*{\section}{0pt}{1.1em}{0.45em}
\titlespacing*{\subsection}{0pt}{0.8em}{0.3em}
\titleformat{\section}{\normalfont\large\bfseries}{\thesection}{0.8em}{}
\titleformat{\subsection}{\normalfont\normalsize\bfseries}{\thesubsection}{0.8em}{}

\theoremstyle{plain}
\newtheorem{theorem}{Theorem}
\newtheorem{proposition}{Proposition}

\newtheorem{corollary}{Corollary}
\theoremstyle{definition}

\newcommand{\E}{\mathbb{E}}
\newcommand{\Prob}{\mathbb{P}}

\newcommand{\Adv}{\hat{A}}

\definecolor{cnavy}{HTML}{1F4E79}
\definecolor{cink}{HTML}{16202B}
\definecolor{cpanel}{HTML}{EEF1F4}
\definecolor{carrow}{HTML}{8A94A0}
\definecolor{cclay}{HTML}{B5654A}

\newtcolorbox{takeaway}{enhanced, center, width=0.92\linewidth, colback=cpanel, colframe=cnavy,
  boxrule=0pt, leftrule=2.4pt, arc=2pt, left=11pt, right=11pt, top=7pt, bottom=7pt,
  fontupper=\small, before skip=0.8\baselineskip, after skip=0.8\baselineskip}

\newcommand{\papertitle}{GRPO, Dr.\,GRPO, and DAPO Are Three Operations on One Number:\\[2pt] The Group-Standard-Deviation Identity}
\hypersetup{pdftitle={GRPO, Dr. GRPO, and DAPO Are Three Operations on One Number: The Group-Standard-Deviation Identity},
  pdfauthor={Yong Yi Bay and Kathleen A. Yearick},
  pdfkeywords={GRPO, reward normalization, standard deviation, silent groups, difficulty bias, group size, dynamic sampling, RLVR, LLM reasoning}}
\newsavebox{\papertitlebox}
\newcommand{\paperfrontmatter}{%
\begin{center}
\vspace*{-0.62in}
\begin{lrbox}{\papertitlebox}%
\begin{varwidth}{\textwidth}
\centering
{\Large\bfseries \papertitle\par}
\end{varwidth}%
\end{lrbox}%
\rule{\wd\papertitlebox}{1.1pt}\par
\vspace{0.55em}
\usebox{\papertitlebox}\par
\vspace{0.55em}
\rule{\wd\papertitlebox}{1.1pt}\par
\vspace{1.7em}
{\textbf{Yong Yi Bay}\textsuperscript{*}\hspace{2.4em}\textbf{Kathleen A. Yearick}\textsuperscript{*}\par}
\vspace{0.5em}
PhD, University of Illinois at Urbana-Champaign\par
\let\thefootnote\relax\footnotetext{\textsuperscript{*}Equal contribution. Correspondence: \texttt{\{yongyibay, kallie.a.yearick\}@gmail.com}.}
\vspace{1.9em}
\end{center}
}

\newenvironment{paperabstract}
{\begin{center}\textbf{\textsc{\textls[120]{Abstract}}}\end{center}
\begin{list}{}{\leftmargin=0.78in\rightmargin=0.78in}\item\small}
{\end{list}\normalsize}

\newcommand{\paperkeywords}[1]{\noindent\textbf{\textit{Keywords}} #1\par\vspace{0.65em}}

\begin{document}
\paperfrontmatter

\begin{paperabstract}
Three of the most popular methods for training language models to reason look like three different tricks. They are not. All three adjust a single number: \emph{standard deviation}, reflecting how much a prompt's sampled answers disagree. When such a model is trained, it answers each problem many times, and an automatic checker marks every answer right or wrong. The standard deviation of those marks measures the disagreement: largest when the answers split evenly between right and wrong, and zero when they all agree. Group Relative Policy Optimization (GRPO) divides by this number, GRPO Done Right (Dr.\,GRPO) drops the division, and Decoupled Clip and Dynamic Sampling Policy Optimization (DAPO) discards the groups where it is zero. Each is presented as its own fix, yet this paper proves they are three settings of one dial. That dial is not cosmetic: for right-or-wrong rewards, the disagreement is exactly the size of the training update, the \emph{group-standard-deviation identity}. A split group teaches the most, while a unanimous group teaches nothing and falls silent. The same result says which problems deserve the most weight and how many tries each one needs. This paper confirms the intuition on a large real difficulty dataset (Big-Math) and in a controlled training run. What looks like a harmless normalization step is the dial that decides where learning happens and how strongly.
\end{paperabstract}

\paperkeywords{GRPO $\cdot$ reward normalization $\cdot$ silent groups $\cdot$ difficulty bias $\cdot$ group size $\cdot$ dynamic sampling $\cdot$ RLVR $\cdot$ LLM reasoning}

\section{Introduction}

Teaching a language model to reason starts with a step that looks wasteful: the same prompt is answered many times over, on purpose. This happens during reinforcement learning, not when a user chats with the model. The model produces a group of candidate answers, an external verifier marks each one as correct or incorrect, and the optimizer updates the model so that rewarded answer paths become more likely. The repeated answers are not redundant outputs for a user; they are measurements of the model's current uncertainty on that prompt.

A simple example captures the mechanism. Suppose a model attempts the same math problem eight times. If all eight attempts are wrong, there is no successful attempt to imitate. If all eight are right, there is no failed attempt to move away from. The useful training case is mixed: some attempts are right and some are wrong. Only then can the training rule compare the two sides. In this sense, a prompt teaches through its \emph{within-group disagreement}.

Group Relative Policy Optimization \citep[GRPO;][]{shao2024deepseekmath,deepseekr1}, the workhorse of current verifiable reasoning training, is built around this comparison. GRPO operates in the \emph{reinforcement learning with verifiable rewards} setting (RLVR), where an automatic checker returns a reward, usually $1$ for a correct final answer and $0$ for an incorrect one. The model supplies the candidate answers; the verifier supplies the rewards; GRPO supplies the rule that converts those rewards into \emph{advantages}; the optimizer changes the model parameters.

\begin{figure}[H]
\centering
\resizebox{0.96\linewidth}{!}{%
\begin{tikzpicture}[
  font=\small,
  box/.style={draw=cnavy, fill=cpanel, line width=0.9pt, text=cink, align=center,
              inner xsep=11pt, inner ysep=9pt, minimum width=52mm, minimum height=16mm,
              rounded corners=3pt},
  hub/.style={draw=cnavy, fill=cnavy, text=white, line width=1.0pt, align=center,
              inner xsep=11pt, inner ysep=9pt, minimum width=52mm, minimum height=16mm,
              rounded corners=3pt},
  lbl/.style={text=cnavy, font=\footnotesize, align=center, inner sep=3.5pt},
  flow/.style={-{Stealth[length=3.2mm,width=2.6mm]}, draw=carrow, line width=1.2pt,
               shorten >=3.5pt, shorten <=3.5pt}
]
\node[box] (policy)   at (0,3.5)   {Policy $\pi_\theta$ samples\\$G$ answers to prompt $x$};
\node[box] (verifier) at (9.8,3.5) {Verifier marks each\\answer right or wrong, $R_i\in\{0,1\}$};
\node[hub] (sigma)    at (9.8,0)   {Group disagreement\\$\sigma=\sqrt{k(G-k)}/G$};
\node[box] (update)   at (0,0)     {Training rule updates $\theta$\\GRPO\;/\;Dr.\,GRPO\;/\;DAPO};
\draw[flow] (policy)   -- node[lbl,above]{$G$ answers $y_1,\dots,y_G$} (verifier);
\draw[flow] (verifier) -- node[lbl,right]{rewards $R_1,\dots,R_G$}      (sigma);
\draw[flow] (sigma)    -- node[lbl,below]{advantages $\Adv_i$}          (update);
\draw[flow] (update)   -- node[lbl,left,align=center]{updated model,\\next step} (policy);
\end{tikzpicture}}
\caption{The training-time loop studied in this paper. The trainer samples one prompt many times to compare its correct and incorrect attempts, and the group reward standard deviation $\sigma$ measures whether they disagree.}
\label{fig:trainingloop}
\end{figure}
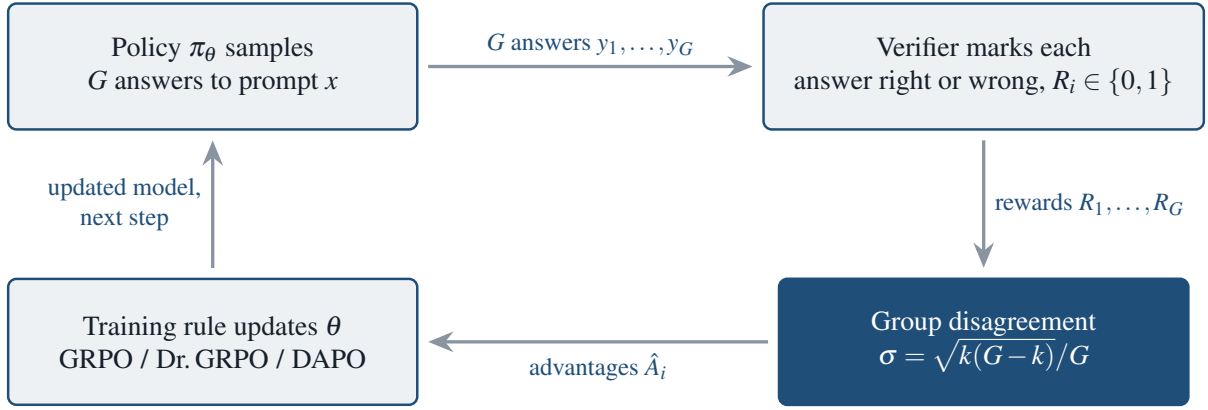

The mean subtraction inside GRPO needs little controversy: subtracting any action-independent baseline preserves the policy gradient while reducing variance \citep{williams1992reinforce}. The contested step is the next one, division by the group standard deviation. It is often dismissed as a normalization detail, yet \citet{liu2025drgrpo} identify it as the source of a \emph{question-level difficulty bias} and remove it. Its large-group effect is already understood: for binary rewards it makes GRPO ascend not the raw success rate $p$, but the \emph{arcsine transform} $\E_x[2\arcsin\sqrt{p_x}]$, the classical variance-stabilizing transform of a binomial proportion \citep{bartlett1936,anscombe1948}. \citet{thrampoulidis2025advantageshaping} established this surrogate-reward view.

\begin{takeaway}
\textbf{The lens.} The central object is the group reward standard deviation $\sigma$. It is the amount of disagreement in the verifier's marks inside one prompt's sampled group. GRPO divides by $\sigma$, Dr.\,GRPO drops that division, and DAPO's dynamic sampling discards groups with $\sigma=0$. The paper's claim is that these are not separate tricks but three operations on the same number.
\end{takeaway}

\paragraph{This paper.} Real training does not take the large-group limit; it chooses a finite group size $G$ and updates from the sampled group it receives. The paper's main claim is the exact finite-group accounting behind that update. For binary rewards, if $k$ of $G$ sampled answers are correct, then a single GRPO step on that prompt is, exactly and in any dimension,
\begin{equation}
\label{eq:headline}
g\;=\;\frac1G\sum_i \Adv_i\,\nabla_\theta\log\pi_\theta(y_i)\;=\;\sigma\,\big(\bar s_+-\bar s_-\big),
\qquad \sigma=\frac{\sqrt{k(G-k)}}{G},\quad 0<k<G,
\end{equation}
\emph{independently of the baseline}, where $\bar s_+$ and $\bar s_-$ are the mean scores of the correct and incorrect rollouts. The update has two plain pieces. The direction $\bar s_+-\bar s_-$ says what to favor: correct rollouts over incorrect rollouts. The multiplier $\sigma$ says how strongly to favor it: it vanishes for a unanimous group and peaks for an evenly split one. Equation~\eqref{eq:headline} is the \emph{group-standard-deviation identity}: what sits in the advantage's denominator is the length of the gradient itself. The scalar case $g(k)=\sqrt{k(G-k)}/G$ (where $\bar s_+-\bar s_-=1$) is the form used in the remaining analysis. Averaging \eqref{eq:headline} over groups recovers the arcsine gradient of \citet{thrampoulidis2025advantageshaping} as $G\to\infty$; the identity is the exact finite-$G$ object underneath that limit. Figure~\ref{fig:trainingloop} locates the quantity inside the training loop, and Figure~\ref{fig:schematic} gives the resulting method map: GRPO, Dr.\,GRPO, and DAPO act on one scalar quantity in three different ways.

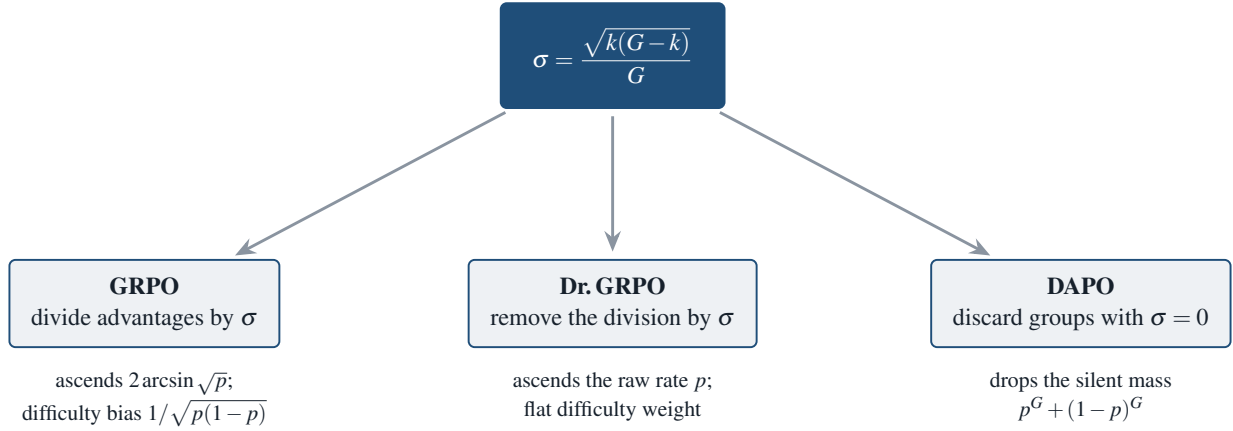
\begin{figure}[H]
\centering
\resizebox{0.98\linewidth}{!}{%
\begin{tikzpicture}[
  font=\small,
  hub/.style={draw=cnavy, fill=cnavy, text=white, line width=0.8pt, align=center,
              inner xsep=13pt, inner ysep=10pt, rounded corners=2.5pt},
  box/.style={draw=cnavy, fill=cpanel, line width=0.8pt, text=cink, align=center,
              inner xsep=9pt, inner ysep=7pt, minimum width=33mm, rounded corners=2.5pt},
  eff/.style={text=cink, align=center, font=\footnotesize},
  note/.style={text=cnavy, align=center, font=\footnotesize\itshape},
  flow/.style={-{Stealth[length=3mm, width=2.4mm]}, line width=1.2pt, draw=carrow,
               shorten >=3pt, shorten <=3pt}
]
\node[hub] (sig) at (0,0) {$\displaystyle \sigma=\frac{\sqrt{k(G-k)}}{G}$};
\node[note, anchor=south] at (0,1.3) {per-prompt update $g=\sigma\,(\bar s_+-\bar s_-)$: the group's reward std times a right-minus-wrong contrast};
\node[box] (grpo)   at (-6.8,-3.6) {\textbf{GRPO}\\[1pt]divide advantages by $\sigma$};
\node[box] (drgrpo) at ( 0.0,-3.6) {\textbf{Dr.\,GRPO}\\[1pt]remove the division by $\sigma$};
\node[box] (dapo)   at ( 6.8,-3.6) {\textbf{DAPO}\\[1pt]discard groups with $\sigma=0$};
\node[eff, below=6pt of grpo]   {ascends $2\arcsin\sqrt p$;\\[1pt]difficulty bias $1/\sqrt{p(1-p)}$};
\node[eff, below=6pt of drgrpo] {ascends the raw rate $p$;\\[1pt]flat difficulty weight};
\node[eff, below=6pt of dapo]   {drops the silent mass\\[1pt]$p^{G}+(1-p)^{G}$};
\draw[flow] (sig) -- (grpo);
\draw[flow] (sig) -- (drgrpo);
\draw[flow] (sig) -- (dapo);
\end{tikzpicture}}
\caption{One object, three interventions. The group reward standard deviation $\sigma=\sqrt{k(G-k)}/G$ that GRPO divides by, Dr.\,GRPO drops, and DAPO filters to zero.}
\label{fig:schematic}
\end{figure}

\paragraph{Why the exact form matters.} A practitioner does not choose an asymptotic limit; a practitioner chooses $G$ and decides which prompts to keep. The finite-group identity turns both choices into closed forms. The resulting contributions are:
\begin{enumerate}
\item \textbf{The group-standard-deviation identity} (\S\ref{sec:finiteG}, Theorem~\ref{thm:identity}). The per-prompt GRPO update is the contrastive score direction scaled by the group's reward standard deviation $\sqrt{k(G-k)}/G$, exactly, baseline-free, and in any dimension. Its expectation over $k\sim\mathrm{Binomial}(G,p)$ is an exact binomial sum; the large-group arcsine gradient $\sqrt{p(1-p)}$ and its finite-$G$ attenuation $1-1/(8Gp(1-p))$ are its limit and first correction.
\item \textbf{A closed-form group-size law} (\S\ref{sec:groupsize}, Corollary~\ref{cor:designrule}). A group of size $G$ realizes a fraction $\varphi\approx1-1/(8Gp(1-p))$ of the large-group gradient, so reaching fidelity $1-\varepsilon$ needs $G\gtrsim 1/(8\varepsilon p(1-p))$, the square of the difficulty weight: a coin-flip prompt is faithful by $G\approx10$, a prompt at $5\%$ success needs $G\approx70$ (Table~\ref{tab:groupsize}). The budget is read off difficulty, not swept.
\item \textbf{The silent-group rate} (\S\ref{sec:degenerate}). A group is silent (zero advantage everywhere) with probability $p^{G}+(1-p)^{G}$; this is exactly the $\sigma=0$ mass that DAPO's dynamic sampling \citep{yu2025dapo} over-samples and discards, and its logged all-correct fraction is the same functional family $\E_p[p^{n}]$, whose shape and sub-one plateau the closed form reproduces.
\item \textbf{The difficulty bias, exactly} (\S\ref{sec:bias}). The rate at which GRPO converts success probability into objective is $\partial_p\,2\arcsin\sqrt p=1/\sqrt{p(1-p)}$, the very $1/\sigma$ the advantage divides by; this is the question-level difficulty bias of \citet{liu2025drgrpo}, and deleting $\sigma$ reverts the objective from $\arcsin\sqrt p$ to $p$. Group-mean centering is the leave-one-out (RLOO) advantage up to the constant $G/(G-1)$, so the division is the only objective-changing step.
\item \textbf{Validation on real difficulty distributions} (\S\ref{sec:realdata}). On Big-Math \citep{albalak2025bigmath}, $N=215{,}608$ problems with empirical solve rates, the standardization moves $13.9\%\!\to\!24.7\%$ of the implicit objective's gradient mass onto extreme prompts; the silent-group rate is $44\%$ at the common group size $G=8$ and matches direct subsampling of the logged rollouts to within two points.
\item \textbf{The predictions in a controlled run} (\S\ref{sec:experiment}). A real GRPO loop over $6{,}000$ Bernoulli-logit prompts confirms the closed forms as training dynamics: the silent-group rate tracks the measured wasted-group fraction step by step ($R^2=0.999$), the realized gradient mass matches the finite-$G$ reweighting, and the difficulty bias is visible as GRPO lifting the hardest prompts where Dr.\,GRPO stalls.
\end{enumerate}
These results do not require a new algorithm or a large training run. They follow from exact accounting of one sampled group, with the data-dependent quantities read off published rollout statistics.

\section{Three Methods, One Operation Apart}
\label{sec:setup}

This section names the objects in the training loop and spells out the three method formulas. The setting is a single prompt $x$ during training. The policy $\pi_\theta$ samples responses, the verifier assigns rewards, and the training rule converts the reward pattern into a parameter update. All comparisons below differ only in how they handle the group reward standard deviation $\sigma$.

\paragraph{Acronyms and scope.} The three names are Group Relative Policy Optimization (GRPO), GRPO Done Right (Dr.\,GRPO), and Decoupled Clip and Dynamic Sampling Policy Optimization (DAPO). Dr.\,GRPO also discusses length normalization, and DAPO contains additional engineering choices. The present analysis isolates the common axis relevant to all three: the standard-deviation operation on a group of verifier rewards.

\paragraph{GRPO advantages.} For a prompt $x$, GRPO draws a group of $G$ responses $y_1,\dots,y_G\sim\pi_\theta(\cdot\mid x)$ with scalar rewards $R_1,\dots,R_G$, and forms the standardized advantage
\begin{equation}
\label{eq:adv}
\Adv_i \;=\; \frac{R_i-\mu}{\sigma},\qquad
\mu=\frac1G\sum_{j}R_j,\quad
\sigma=\sqrt{\frac1G\sum_j (R_j-\mu)^2},
\end{equation}
which is broadcast to every token of $y_i$ and plugged into the usual clipped surrogate with a KL penalty to a reference policy. The analysis here concerns the advantage construction; the clipping and KL terms are standard regularizers that do not affect the advantage's expectation at the first, on-policy step (ratio $=1$), and are held aside throughout.

\paragraph{Two updates, one difference.} Dropping the division by $\sigma$ from \eqref{eq:adv} leaves the mean-centered advantage $\Adv_i=R_i-\mu$. This second update is neither new nor specific to one method: it is the update used by Dr.\,GRPO \citep{liu2025drgrpo}, and it equals the leave-one-out (RLOO) advantage \citep{ahmadian2024rloo,kool2019buy} and baselined REINFORCE \citep{williams1992reinforce} up to a learning-rate constant (Proposition~\ref{prop:loo}). The equivalence is straightforward: each method subtracts a baseline from the reward and stops there. GRPO alone adds one further operation, division by $\sigma$. Thus the objective-changing distinction is not the baseline; it is the standard-deviation operation. Table~\ref{tab:threeops} makes this explicit after the binary-reward notation is introduced.

\paragraph{Binary rewards.} In RLVR the reward is the verifier's verdict, $R_i\in\{0,1\}$. The verifier is the checker: for math it may compare final answers or symbolic forms, for code it may run tests, and for structured tasks it may apply a parser or schema. GRPO does not decide truth by itself; it receives the reward vector $R_1,\dots,R_G$ from this checker. Write $p=p_x(\theta)=\Prob_{y\sim\pi_\theta(\cdot|x)}[\,R=1\,]$ for the policy's success probability on $x$, and let $k=\sum_i R_i$ be the number of correct samples in a group. Then $\mu=k/G$ and, because rewards are Bernoulli, $\sigma=\sqrt{\mu(1-\mu)}=\sqrt{k(G-k)}/G$ \emph{exactly}. When $k=0$ all sampled answers are wrong; when $k=G$ all sampled answers are right. In both cases $\sigma=0$, so there is no right-versus-wrong contrast inside the group.

\begin{table}[H]
\centering
\small
\setlength{\tabcolsep}{7pt}
\renewcommand{\arraystretch}{1.4}
\begin{tabular}{@{}p{0.12\linewidth}p{0.21\linewidth}p{0.27\linewidth}p{0.27\linewidth}@{}}
\toprule
\textbf{Method} & \textbf{Operation on $\sigma$} & \textbf{Per-prompt update $g$} & \textbf{In words} \\
\midrule
GRPO & Divide by $\sigma$ & $g=\sigma\,\Delta s$ & The update size is the group's disagreement. \\
\addlinespace[4pt]
Dr.\,GRPO & Remove the division & $g=\sigma^{2}\,\Delta s$ & Ascends raw accuracy, with a flat difficulty weight. \\
\addlinespace[4pt]
DAPO & Drop the $\sigma=0$ groups & $g=\mathbf{1}\{0<k<G\}\,\sigma\,\Delta s$ & Discards groups with no right-versus-wrong contrast. \\
\bottomrule
\end{tabular}
\caption{The three methods as operations on one number. Here $\Delta s=\bar s_+-\bar s_-$ and $\sigma=\sqrt{k(G-k)}/G$; the advantage is $\Adv_i=(R_i-\mu)/\sigma$ for GRPO and $\Adv_i=R_i-\mu$ for Dr.\,GRPO. GRPO scales the update by $\sigma$, Dr.\,GRPO removes that scaling, and DAPO changes which groups contribute.}
\label{tab:threeops}
\end{table}

\paragraph{The score identity.} The argument uses the elementary policy-gradient fact that for any baseline $b$ independent of $y$,
\begin{equation}
\label{eq:score}
\E_{y\sim\pi_\theta(\cdot|x)}\!\big[(R(y)-b)\,\nabla_\theta\log\pi_\theta(y\mid x)\big]
=\nabla_\theta\,\E_y[R(y)] = \nabla_\theta\,p_x(\theta).
\end{equation}

\paragraph{Mean-centering is leave-one-out, rescaled.} The mean subtraction has a familiar form. Excluding the sample being scored gives the leave-one-out (RLOO) baseline $b_i=\frac{1}{G-1}\sum_{j\ne i}R_j$ \citep{ahmadian2024rloo,kool2019buy}, whose advantage $R_i-b_i$ is unbiased by \eqref{eq:score}, and subtracting the group mean is the same thing up to a constant.
\begin{proposition}[Group-mean centering is rescaled RLOO]
\label{prop:loo}
For any group $R_1,\dots,R_G$ with $G\ge2$ and RLOO baseline $b_i=\frac{1}{G-1}\sum_{j\ne i}R_j$,
\[
R_i-\mu \;=\; \frac{G-1}{G}\,(R_i-b_i)\qquad\text{for every }i.
\]
\end{proposition}
\begin{proof}
$\sum_j R_j=R_i+(G-1)b_i$, so $\mu=\tfrac1G\big(R_i+(G-1)b_i\big)$ and $R_i-\mu=\tfrac{G-1}{G}(R_i-b_i)$.
\end{proof}
Mean-centering is therefore the unbiased RLOO advantage up to the constant $G/(G-1)$, which a learning rate absorbs; both updates share it, so it is not where they differ. The division by $\sigma$, the step that does change the objective, is analyzed next.

\section{The Update Is the Group's Standard Deviation}
\label{sec:finiteG}

The main result is the mathematical version of the training story above. Take one prompt, sample $G$ answers, and let the verifier split them into correct and incorrect sets. GRPO moves the model toward the correct side and away from the incorrect side. The theorem below states that the length of this move is exactly the group reward standard deviation. Thus the same scalar that appears in the denominator of the advantage is also the scalar that measures how much learning signal the prompt actually produced. Everything downstream (the group-size budget, the silent fraction, and the difficulty bias) is read from this one form. The statement holds for any policy in any dimension, not only for a scalar model.

\begin{theorem}[The group-standard-deviation identity]
\label{thm:identity}
Fix a prompt and a group of $G$ responses with binary rewards, $k$ of them correct, $0<k<G$. Let $s_i=\nabla_\theta\log\pi_\theta(y_i\mid x)$ be the score of response $i$, and let $\bar s_+=\frac1k\sum_{i:R_i=1}s_i$ and $\bar s_-=\frac1{G-k}\sum_{i:R_i=0}s_i$ be the mean scores of its correct and incorrect responses. The per-prompt GRPO update is, exactly,
\begin{equation}
\label{eq:exact}
g\;=\;\frac1G\sum_i \Adv_i\,s_i\;=\;\sigma\,\big(\bar s_+-\bar s_-\big),
\qquad \sigma=\frac{\sqrt{k(G-k)}}{G},
\end{equation}
\emph{independently of any baseline}, and $g=0$ when $k\in\{0,G\}$. The scalar coefficient is the group's empirical reward standard deviation; the direction $\bar s_+-\bar s_-$ contrasts the correct and incorrect responses. For a one-dimensional Bernoulli-logit prompt $p=\varsigma(\theta)$, where $s_i=y_i-p$ and $\bar s_+-\bar s_-=1$, this is the scalar form
\begin{equation}
\label{eq:scalar}
g(k)\;=\;\frac{\sqrt{k(G-k)}}{G}\;=\;\sigma .
\end{equation}
\end{theorem}

\begin{proof}
With binary rewards the standardized advantage takes two values, $\Adv_+=(1-\mu)/\sigma=(G-k)/(G\sigma)$ on the $k$ correct responses and $\Adv_-=-\mu/\sigma=-k/(G\sigma)$ on the $G-k$ incorrect ones, with $\mu=k/G$. Since $\sum_{i:R_i=1}s_i=k\,\bar s_+$ and $\sum_{i:R_i=0}s_i=(G-k)\,\bar s_-$,
\[
g=\frac1G\big(\Adv_+\,k\,\bar s_+ + \Adv_-\,(G-k)\,\bar s_-\big)
=\frac{k(G-k)}{G^2\sigma}\,(\bar s_+-\bar s_-)
=\sigma\,(\bar s_+-\bar s_-),
\]
the last step using $k(G-k)/(G^2\sigma)=\sigma$ from $\sigma=\sqrt{k(G-k)}/G$. Adding any baseline $b$ to the rewards shifts every $\Adv_i$ equally and cancels because $\sum_i\Adv_i=0$. The scalar form follows from $\bar s_+-\bar s_-=(1-p)-(-p)=1$.
\end{proof}

\begin{takeaway}
\textbf{The group-standard-deviation identity.} The update factors cleanly. Its direction $\bar s_+-\bar s_-$ points from the incorrect rollouts toward the correct ones; its length is the group's reward standard deviation $\sigma=\sqrt{k(G-k)}/G$, zero for a unanimous group and largest for an even split. This is the paper's central lens: for binary rewards the standard deviation is not a denominator used for normalization but the prompt's learning signal.
\end{takeaway}

Averaging the identity over random groups gives an \emph{exact} expected gradient, whose limit is the established large-group arcsine reading.

\begin{proposition}[Exact expected gradient]
\label{prop:expectation}
For the scalar prompt \eqref{eq:scalar} with $k\sim\mathrm{Binomial}(G,p)$, the expected per-prompt gradient is the exact finite sum
\begin{equation}
\label{eq:exactE}
\E[g]=\frac1G\sum_{k=1}^{G-1}\binom{G}{k}p^k(1-p)^{G-k}\sqrt{k(G-k)} .
\end{equation}
\end{proposition}

\begin{corollary}[Large-group attenuation, asymptotic]
\label{cor:attenuation}
As $G\to\infty$,
\begin{equation}
\label{eq:attenuation}
\E[g]=\sqrt{p(1-p)}\left(1-\frac{1}{8\,G\,p(1-p)}+O(G^{-2})\right),
\end{equation}
the arcsine gradient $\sqrt{p(1-p)}$ attenuated by the relative factor $1/(8Gp(1-p))$. This is the interior expansion of the exact sum \eqref{eq:exactE}; it loses accuracy near $p\in\{0,1\}$, where the unanimous mass $p^{G}+(1-p)^{G}$ is not negligible. At $G=8,\,p=0.05$ the exact realized fraction $\E[g]/\sqrt{p(1-p)}$ is $0.54$, against the expansion's $0.67$.
\end{corollary}
\begin{proof}
A second-order delta-method expansion of $f(\mu)=\sqrt{\mu(1-\mu)}$ about $\mu=p$, with $f''(p)=-\tfrac14(p(1-p))^{-3/2}$ and $\operatorname{Var}(\mu)=p(1-p)/G$, gives $\E[f(\mu)]\approx\sqrt{p(1-p)}-\tfrac{1}{8G}(p(1-p))^{-1/2}$; the excluded unanimous terms are $O(p^{G}+(1-p)^{G})$.
\end{proof}

The number $\sigma$ is computed from the sampled group alone; it does not require knowledge of the policy's true success probability $p$. It is therefore an observable training diagnostic: before estimating any global difficulty, the sampled rollouts already reveal how much signal the prompt produced. Figure~\ref{fig:identity} (left) confirms the scalar form: the gradient lands on $\sqrt{k(G-k)}/G$, and the Monte-Carlo markers at three baselines coincide because the baseline cancels in \eqref{eq:exact}. Averaging over groups gives the arcsine gradient $\sqrt{p(1-p)}$, attenuated at finite $G$ by Corollary~\ref{cor:attenuation}. The identity itself is the finite-$G$ statement used by an actual training step.

\begin{figure}[H]
\centering
\includegraphics[width=\textwidth]{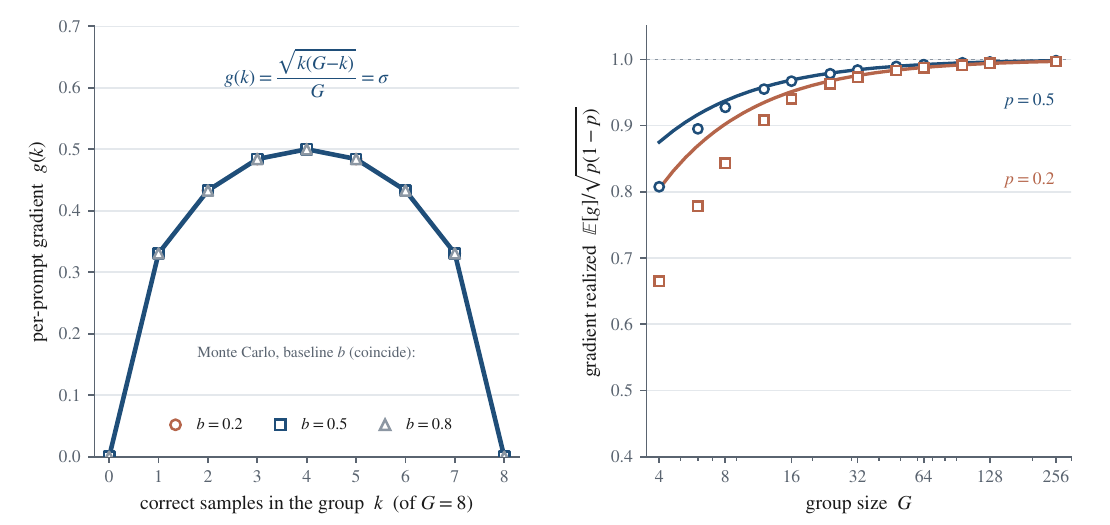}
\caption{The group-standard-deviation identity. \emph{Left:} the per-prompt gradient at $G=8$, zero when all samples agree and largest at an even split, landing on $\sqrt{k(G-k)}/G$; Monte-Carlo markers at three baselines $b\in\{0.2,0.5,0.8\}$. \emph{Right:} the realized fraction $\E[g]/\sqrt{p(1-p)}$ against the asymptotic law $1-1/(8Gp(1-p))$ of Corollary~\ref{cor:attenuation}.}
\label{fig:identity}
\end{figure}

\section{How Many Samples a Prompt Needs}
\label{sec:groupsize}

Group size is usually fixed before training, often by convention ($G=8$ or $16$) rather than by calculation. Corollary~\ref{cor:attenuation} turns this choice into a difficulty-dependent quantity. The question becomes: how many samples are needed before a finite group behaves like the large-group limit? Define the \emph{gradient fidelity}
\begin{equation}
\label{eq:fidelity}
\varphi(G,p)\;=\;\frac{\E[g]}{\sqrt{p(1-p)}}\;\in\;[0,1],
\end{equation}
as the fraction of the large-group (arcsine) gradient realized by a group of size $G$ at difficulty $p$. It approaches $1$ as $G\to\infty$, and the interior expansion \eqref{eq:attenuation} gives $\varphi\approx 1-1/(8Gp(1-p))$. Solving this expression for $G$ gives the group-size law.

\begin{corollary}[The group-size law]
\label{cor:designrule}
A group of size $G$ realizes interior fidelity $\varphi\ge 1-\varepsilon$ at difficulty $p$ once
\begin{equation}
\label{eq:designrule}
G\;\ge\;G^\star(\varepsilon,p)\;=\;\frac{1}{8\,\varepsilon\,p(1-p)}\;=\;\frac{w(p)^2}{8\,\varepsilon},
\qquad w(p)=\frac{1}{\sqrt{p(1-p)}} .
\end{equation}
The budget is the mid-difficulty cost $1/(2\varepsilon)$ times a difficulty penalty $1/[4p(1-p)]\ge1$: the requirement grows as the \emph{square} of the difficulty weight $w(p)$ that returns as GRPO's reweighting in \S\ref{sec:bias}. The penalty is $1$ at $p=\tfrac12$, $2.8$ at $p=0.1$, and $5.3$ at $p=0.05$. Near $p\in\{0,1\}$ the unanimous mass $p^{G}+(1-p)^{G}$ (\S\ref{sec:degenerate}) costs more than the interior expansion accounts for, so $G^\star$ understates the requirement there; the exact group size sits beside $G^\star$ in Table~\ref{tab:groupsize}.
\end{corollary}

\begin{table}[H]
\centering
\caption{The group-size law, exact against closed form. Each entry is the exact group size at which the gradient fidelity $\varphi$ of \eqref{eq:fidelity} reaches the column target, with the closed-form budget $G^\star=1/(8\varepsilon p(1-p))$ of \eqref{eq:designrule} in parentheses.}
\label{tab:groupsize}
\small
\input{tables/generated_groupsize_table.tex}
\end{table}

The law can be read from a table rather than swept. Three cases summarize the effect. A coin-flip prompt is sample-efficient: $G=11$ already realizes $95\%$ of the large-group gradient, and the conventional $G=8$ realizes $93\%$. A very hard prompt is more expensive: a prompt solved $5\%$ of the time needs $G=69$ for the same $95\%$ fidelity, roughly six times the coin-flip budget, and at $G=8$ realizes only about half of its large-group gradient ($\varphi=0.54$). This is the same lens again: if the group rarely contains both right and wrong answers, more samples are needed before the standard deviation reveals a stable learning signal. The closed form is most accurate in the high-fidelity regime that forces large groups: at $99\%$ fidelity and $p=0.05$, $G^\star$ gives $263$ against the exact requirement $273$. The discrepancy is concentrated at small groups near the extremes, where unanimous groups remain common. Figure~\ref{fig:groupsize} shows both views: fidelity rising to one with $G$ (left), and the budget's difficulty bathtub (right). The practical reading is simple: one uniform group size spends too much on mid-difficulty prompts and too little on the hard tails, precisely where the silent-group rate of \S\ref{sec:degenerate} is largest. The same budgeting question returns one stage later at inference, governed by a different mechanism: when a fixed prompt is answered by drawing many responses and returning one, the limiting factor is not within-group difficulty but how strongly the responses agree, so the independent-draw accounting here gives way to a correlation ceiling that caps selection from extra samples even while coverage keeps climbing \citep{bay2026ceilings}.

\begin{figure}[H]
\centering
\includegraphics[width=\textwidth]{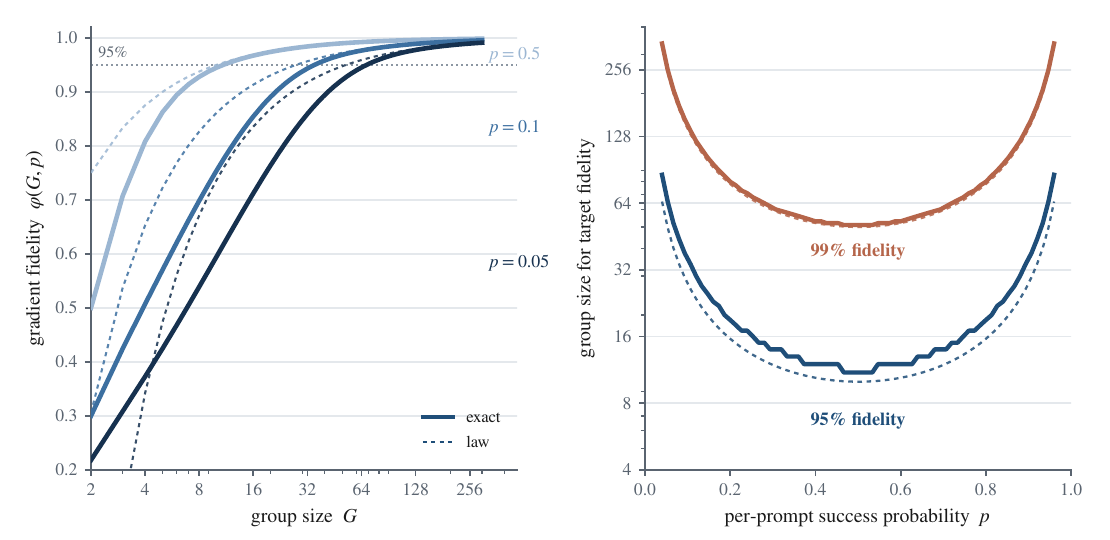}
\caption{The group-size law. \emph{Left:} the gradient fidelity $\varphi(G,p)=\E[g]/\sqrt{p(1-p)}$ against group size for four difficulties, with the closed form $1-1/(8Gp(1-p))$ dashed. \emph{Right:} the group size required for $95\%$ and $99\%$ fidelity against difficulty, exact (solid) and the law $G^\star$ (dashed).}
\label{fig:groupsize}
\end{figure}

\section{Silent Groups, and What DAPO Discards}
\label{sec:degenerate}

The identity \eqref{eq:exact} sets $g=0$ exactly when a group is unanimous. A \emph{silent group} does not mean that the prompt is unimportant; it means that this particular sampled group produced no contrast between right and wrong responses. In the student analogy, all attempts are either failures or successes, so there is no within-prompt comparison to learn from. Thus the second practical decision, which groups to keep, also has a closed form. If every sampled answer is correct or every sampled answer is wrong, the reward standard deviation is zero, the standardized advantage is undefined, and the prompt contributes no gradient. The probability of this event, the \emph{silent-group rate}, is
\begin{equation}
\label{eq:degenerate}
\Prob[\text{group silent}] \;=\; p^{G}+(1-p)^{G},
\end{equation}
which is large for easy or hard prompts. At any interior $p$ it decays geometrically in $G$, but at $p\in\{0,1\}$ it is pinned at $1$: no group size can produce signal when the policy is always wrong or always right (Figure~\ref{fig:finiteG}). This is exactly the failure mode targeted by DAPO's \emph{dynamic sampling} \citep{yu2025dapo}: DAPO over-samples and discards prompts whose group accuracy is $0$ or $1$, keeping only groups with $0<k<G$. In this notation, dynamic sampling is simply a keep rule on the same scalar: retain the group when $\sigma>0$ and replace it when $\sigma=0$. Discarding is not the only response to a silent group: a fixed-reference sign advantage instead assigns a nonzero update to a unanimous group, scoring each response against a constant rather than the group mean \citep{nie2026gradient}.

\begin{figure}[H]
\centering
\includegraphics[width=0.86\textwidth]{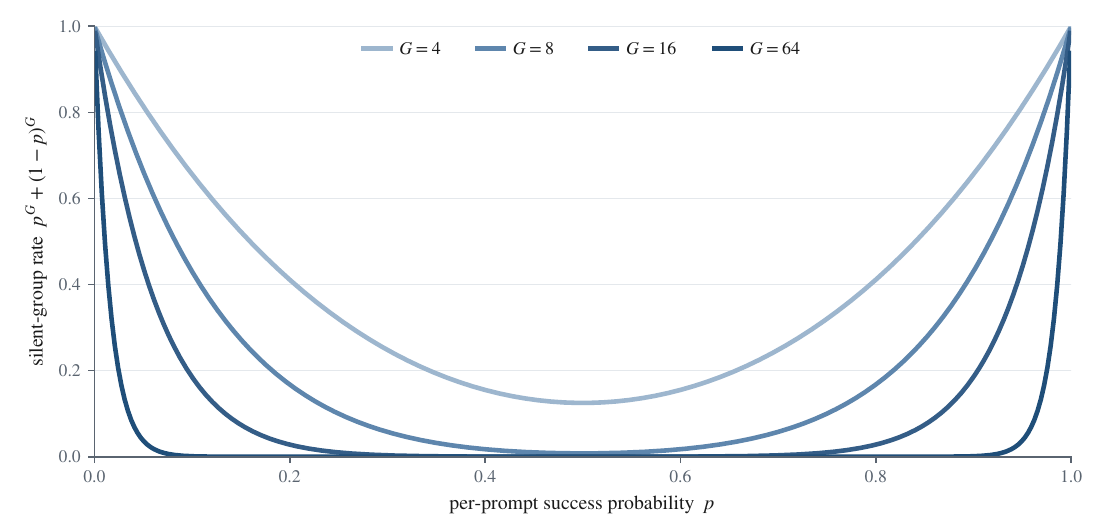}
\caption{The silent-group rate $p^{G}+(1-p)^{G}$ of \eqref{eq:degenerate}. Groups are most often silent near the easy and hard extremes, where all sampled answers tend to share the same reward. Larger $G$ reduces the silent mass in the interior but cannot remove the endpoints $p=0$ and $p=1$.}
\label{fig:finiteG}
\end{figure}

\paragraph{The same accounting on a real run.} The comparison with DAPO has two levels. The first is structural and requires no fitting. DAPO \citep{yu2025dapo} keeps only groups with $0<k<G$, so the discarded groups are exactly the silent mass in \eqref{eq:degenerate}. The fraction logged in Fig.~3b of DAPO (the avg@32 all-correct rate) is the all-correct component $\E_p[p^{32}]$, and the corresponding training-time discarded component at $G=16$ is $\E_p[p^{16}]$. The second level concerns the curve over training time. As training moves the Big-Math difficulty distribution toward mastery and the mean solve rate increases, $\E_p[p^{32}]$ traces the shape of the logged curve with one free timescale, at $R^2=0.92$ (bootstrap $0.91$--$0.96$; Figure~\ref{fig:dapo}, left). This is a consistency check rather than a unique fit: a generic two-parameter saturating curve fits the same points at $R^2=0.98$. The closed form nevertheless explains the qualitative shape without another mechanism. The discarded fraction rises because more prompts become all-correct, yet it saturates below one because a persistent contested mass remains (Figure~\ref{fig:dapo}, right).

\begin{figure}[H]
\centering
\includegraphics[width=\textwidth]{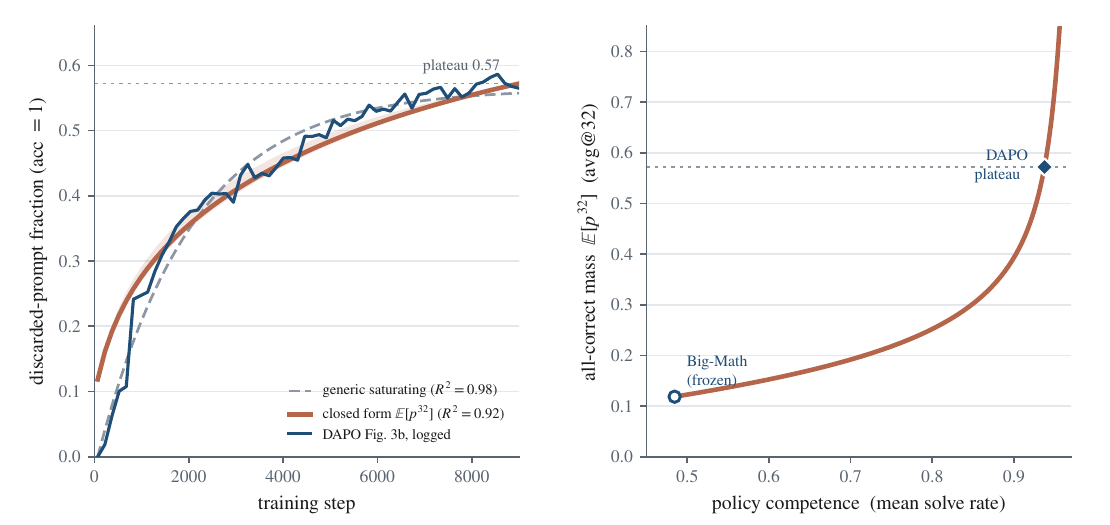}
\caption{DAPO's discarded-prompt fraction and the closed form. \emph{Left:} DAPO's logged accuracy-$1$ fraction (avg@32, Fig.~3b \citep{yu2025dapo}), the closed-form $\E_p[p^{32}]$ fit with bootstrap $10$--$90\%$ band, and a generic saturating baseline. \emph{Right:} the closed-form all-correct mass $\E_p[p^{32}]$ versus policy competence on Big-Math, with the frozen anchor and DAPO's plateau.}
\label{fig:dapo}
\end{figure}

\section{What the Division Optimizes}
\label{sec:bias}

Averaging the identity over groups recovers the known large-group picture and makes the difference between GRPO and Dr.\,GRPO exact. The single operation ``divide by $\sigma$'' does not merely rescale a step; it changes the implicit objective. By \eqref{eq:exactE}, as $G\to\infty$ the per-prompt GRPO gradient tends to $\sqrt{p(1-p)}$, the gradient of $2\arcsin\sqrt{p}$; Dr.\,GRPO tends to $p(1-p)$, the gradient of $p$. This is the surrogate-reward reading \citep{thrampoulidis2025advantageshaping}: GRPO ascends $\E_x[2\arcsin\sqrt{p_x}]$, while Dr.\,GRPO ascends the raw success rate $\E_x[p_x]$. Both objectives have the same endpoint (all prompts solved), but they allocate training pressure differently before that endpoint is reached. Figure~\ref{fig:objective} plots both: the left panel contrasts the two objectives, and the right overlays Monte-Carlo gradients on the closed forms. The remaining gap on the GRPO curve is precisely the finite-$G$ attenuation of \eqref{eq:exactE}.

\paragraph{The difficulty bias is the derivative.} \citet{liu2025drgrpo} (Dr.\,GRPO) observed empirically that standard-deviation normalization induces a ``question-level difficulty bias'': very easy and very hard prompts are over-weighted relative to medium ones. Removing the normalization removes that bias. The large-group limit makes the mechanism exact. The weight GRPO places on an incremental improvement at difficulty $p$ is the derivative of the transform,
\begin{equation}
\label{eq:weight}
w(p)\;=\;\frac{\partial}{\partial p}\,2\arcsin\sqrt p\;=\;\frac{1}{\sqrt{p(1-p)}},
\end{equation}
the same $1/\sigma$ that appears in the advantage \eqref{eq:adv}. This weight has a bathtub shape (Figure~\ref{fig:weight}): it is smallest at $p=\tfrac12$, where $w=2$, and diverges like $p^{-1/2}$ and $(1-p)^{-1/2}$ near $p=0$ and $p=1$. A prompt solved $5\%$ or $95\%$ of the time receives $w\approx4.6$, more than twice the weight of a coin-flip prompt. Dr.\,GRPO sets $w\equiv1$, which integrates back to the raw objective $p$. The difficulty bias is therefore not an accidental side effect. It is the price of using the variance-stabilizing transform: estimator variance is equalized by assigning extra weight to the extremes, where the raw success rate changes slowly.

\begin{figure}[H]
\centering
\includegraphics[width=\textwidth]{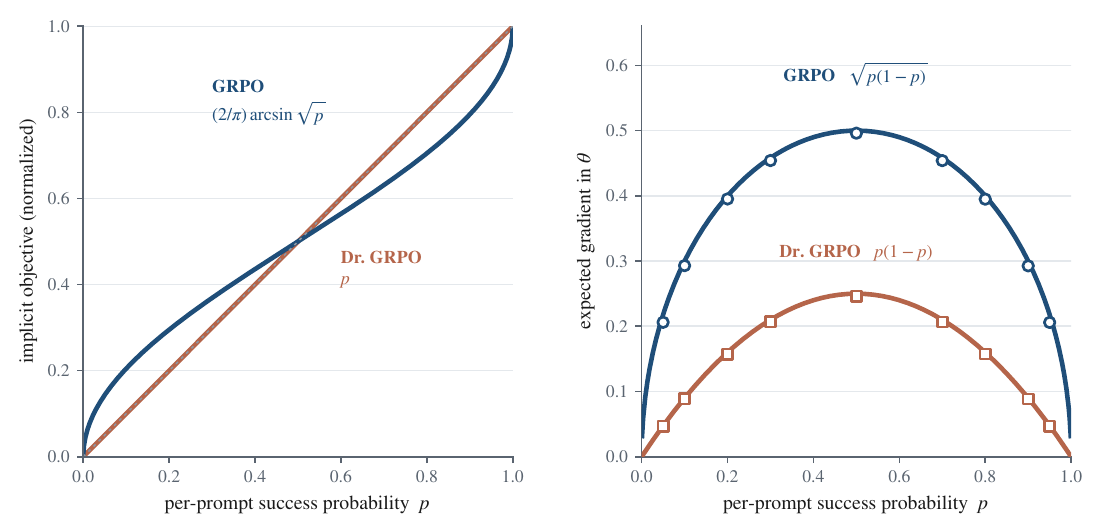}
\caption{The large-group limit of the identity. \emph{Left:} raw success rate $p$ (Dr.\,GRPO) against the arcsine transform $\frac2\pi\arcsin\sqrt p$ (GRPO). \emph{Right:} expected per-prompt gradient: closed-form curves with Monte-Carlo markers over groups of size $G=64$ on a Bernoulli-logit prompt.}
\label{fig:objective}
\end{figure}

\begin{figure}[H]
\centering
\includegraphics[width=0.86\textwidth]{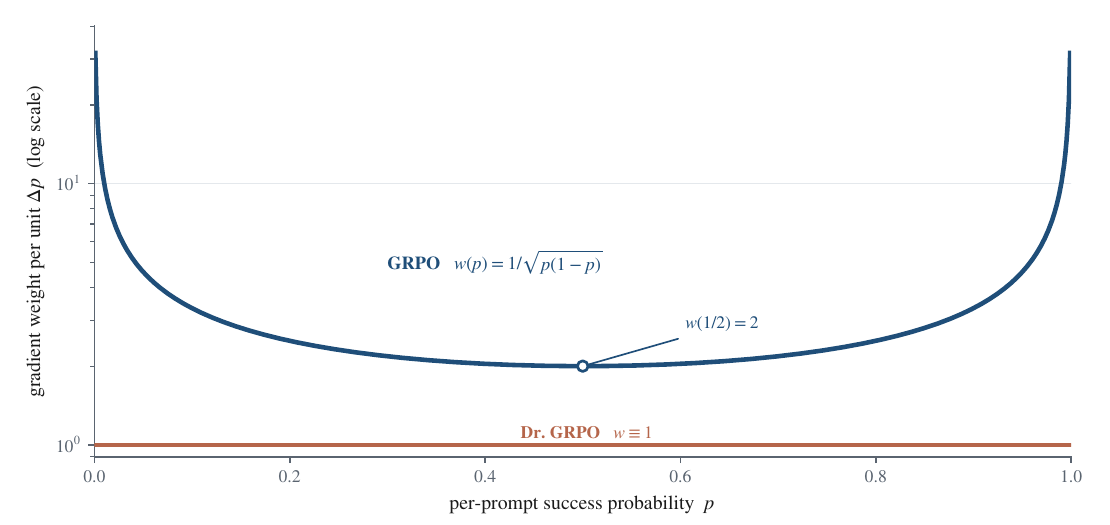}
\caption{The difficulty weight $w(p)=\partial_p\,2\arcsin\sqrt p=1/\sqrt{p(1-p)}$ of \eqref{eq:weight} (log scale), against the flat weight $w\equiv1$ of Dr.\,GRPO. GRPO assigns extra marginal weight to the easiest and hardest prompts, while Dr.\,GRPO weights each unit of raw success-rate improvement equally.}
\label{fig:weight}
\end{figure}

\section{The Lens on Real Difficulty Data}
\label{sec:realdata}

The closed forms above are functions of difficulty $p$, so their practical effect depends on the \emph{distribution} of $p$ in real data. Big-Math \citep{albalak2025bigmath} provides such a distribution: $N=215{,}608$ competition and textbook problems, each annotated with the empirical solve rate $\hat p=k/64$ of Llama-3.1-8B over $64$ rollouts.\footnote{The ungated \texttt{open-r1/Big-Math-RL-Verified-Processed} mirror is used. Llama-3.1-8B's per-prompt solve rate is treated as a stand-in for a policy's success probability $p_x$; it is a realistic, fixed difficulty distribution, not a specific training run. All numbers below are exact functions of the published $\hat p$ histogram.} The distribution (Figure~\ref{fig:realdata}, left) is sharply bimodal: $4.0\%$ of problems are never solved ($\hat p=0$) and $7.2\%$ are always solved ($\hat p=1$), with broad mass in between. Thus the extreme-difficulty regime is not a theoretical corner case; it occupies a visible part of the corpus, precisely where reweighting and silent groups matter most.

\paragraph{The reweighting, quantified.} In the large-group limit \eqref{eq:attenuation}, each prompt contributes per-prompt gradient mass $\propto\sqrt{p(1-p)}$ under GRPO and $\propto p(1-p)$ under Dr.\,GRPO, the gradients of the two implicit objectives. Normalizing each over the corpus gives the share of the total gradient budget spent at each difficulty (Figure~\ref{fig:realdata}, right, and Table~\ref{tab:mass}). Relative to Dr.\,GRPO, GRPO's standardization nearly doubles the share allocated to extreme-difficulty prompts ($13.9\%\to24.7\%$, a factor $1.78$) and reduces the share allocated to medium-difficulty prompts ($22.8\%\to17.5\%$). At finite $G$ the realized shift is milder, because the attenuation of Corollary~\ref{cor:attenuation} discounts the same extremes where silent groups are common; the controlled run of \S\ref{sec:experiment} measures $14.3\%\to17.0\%$ at $G=8$, approaching the large-group $24.7\%$ as $G$ grows. In either view, the difficulty bias is not marginal on this corpus; it reallocates a visible share of the training signal.

\paragraph{Silent groups, quantified.} Using the silent-group rate \eqref{eq:degenerate} on the same $\hat p$ histogram, Table~\ref{tab:degenerate} reports the fraction of prompts that yield no signal as a function of $G$. At the common choice $G=8$, $44\%$ of prompts produce no GRPO gradient; even at $G=64$, $17\%$ do not. An irreducible $11.2\%$ ($\hat p\in\{0,1\}$) are silent at every $G$: no amount of additional sampling creates a right-versus-wrong contrast for prompts that are always wrong or always right under the logged policy. This is the mass that dynamic sampling \citep{yu2025dapo} must replace by over-sampling. The closed form is not only an assumption: drawing size-$G$ groups directly from the $64$ logged rollouts per problem, with no Bernoulli model, gives a measured silent fraction within two points of $p^{G}+(1-p)^{G}$ for $G$ well below the rollout budget ($43\%$ against $44\%$ at $G=8$; Table~\ref{tab:degenerate}, lower row), the small gap being the finite-pool correction.

\begin{figure}[H]
\centering
\includegraphics[width=\textwidth]{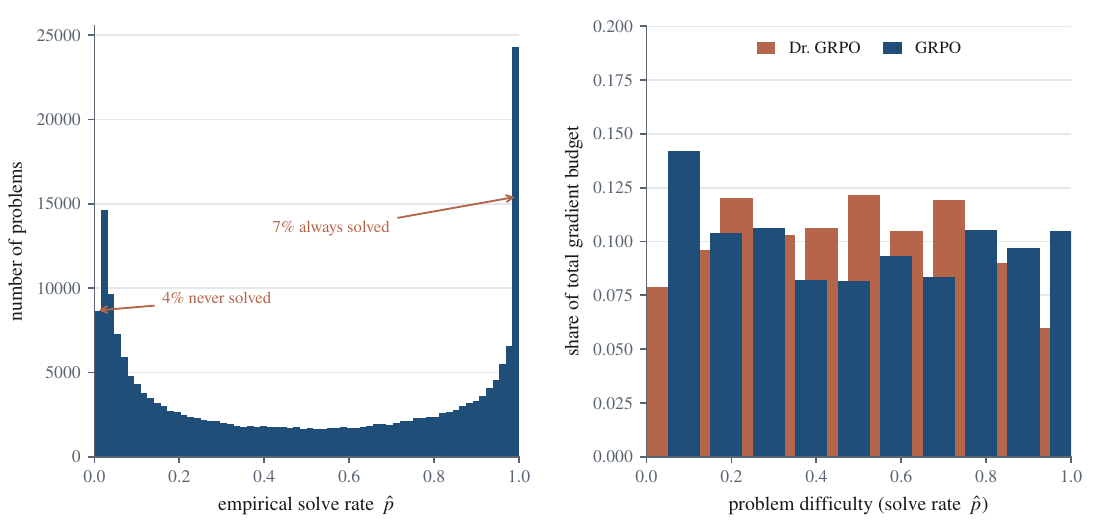}
\caption{Real difficulty distribution and the standardization reweighting (Big-Math, $N=215{,}608$). \emph{Left:} histogram of empirical solve rates $\hat p$ (Llama-3.1-8B, 64 rollouts). \emph{Right:} share of per-prompt gradient budget by difficulty under Dr.\,GRPO ($\propto p(1-p)$) and GRPO ($\propto\sqrt{p(1-p)}$).}
\label{fig:realdata}
\end{figure}

\begin{table}[t]
\centering
\begin{minipage}[t]{0.52\textwidth}\centering
\caption{Share of total gradient mass by prompt difficulty on Big-Math, GRPO versus Dr.\,GRPO.}
\label{tab:mass}
\small
\input{tables/generated_realdata_mass_table.tex}
\end{minipage}\hfill
\begin{minipage}[t]{0.45\textwidth}\centering
\caption{Fraction of Big-Math prompts whose group is silent ($\sigma=0$), by group size $G$: the closed form \eqref{eq:degenerate} against direct subsampling of the $64$ logged rollouts.}
\label{tab:degenerate}
\small
\input{tables/generated_realdata_degenerate_table.tex}
\end{minipage}
\end{table}

\section{The Lens in a Live Training Run}
\label{sec:experiment}

The preceding sections use static accounting: for a fixed difficulty $p$, the formulas predict gradient mass, silent groups, and group-size requirements. This section checks whether the same accounting remains visible during training. The setup is fully reproducible: $M=6{,}000$ prompts, each a one-dimensional Bernoulli-logit policy $p_x=\varsigma(\theta_x)$ with initial difficulty drawn from the real Big-Math solve-rate distribution, are trained for $150$ steps at group size $G=8$ under three advantage rules: GRPO (divide by $\sigma$), Dr.\,GRPO (do not divide), and DAPO (resample degenerate groups). At each step every prompt draws a fresh group, forms the advantage, and updates $\theta_x$ by the actual sampled gradient. Three predictions are compared with the measured run in Figure~\ref{fig:experiment}.

\paragraph{The silent-group rate predicts wasted groups.} Across the whole run the measured fraction of unanimous groups tracks the closed form $\E[p^{G}+(1-p)^{G}]$ evaluated at the current difficulty, with $R^2=0.999$ (Figure~\ref{fig:experiment}a). The fraction \emph{rises} toward one as training proceeds, because mastered prompts ($p\to1$) increasingly become all-correct. The same all-correct mass that climbs in DAPO's logged run (\S\ref{sec:degenerate}) is reproduced here from the identity alone.

\paragraph{The reweighting predicts which difficulties dominate.} Binning the realized gradient mass by difficulty, GRPO and Dr.\,GRPO match their finite-$G$ closed forms exactly (Figure~\ref{fig:experiment}b): GRPO spends $17.0\%$ of its gradient mass on the extreme prompts ($\hat p<0.1$ or $>0.9$) against Dr.\,GRPO's $14.3\%$. This is the finite-$G$ realized version of the large-group $24.7\%$ of \S\ref{sec:realdata}; the gap is the attenuation of Corollary~\ref{cor:attenuation}, which discounts the very extremes where silent groups concentrate.

\paragraph{The difficulty bias is visible in the trajectories.} Because the GRPO step scales like $\sigma$ while the Dr.\,GRPO step scales like $\sigma^{2}$, GRPO moves relatively faster where $\sigma$ is small, namely near the unanimous extremes. The effect is dynamic (Figure~\ref{fig:experiment}c): GRPO lifts the initially-hardest quartile of prompts to a mean solve rate of $0.99$ over the run, whereas Dr.\,GRPO reaches $0.88$; DAPO, which never spends an update on a silent group, is fastest, at the cost of $3.5\times$ oversampling concentrated on those same extremes. Thus the bias is not only a static reallocation of gradient budget. It changes the learning trajectory, especially for prompts that begin near the hard end of the difficulty scale.

\begin{figure}[H]
\centering
\includegraphics[width=\textwidth]{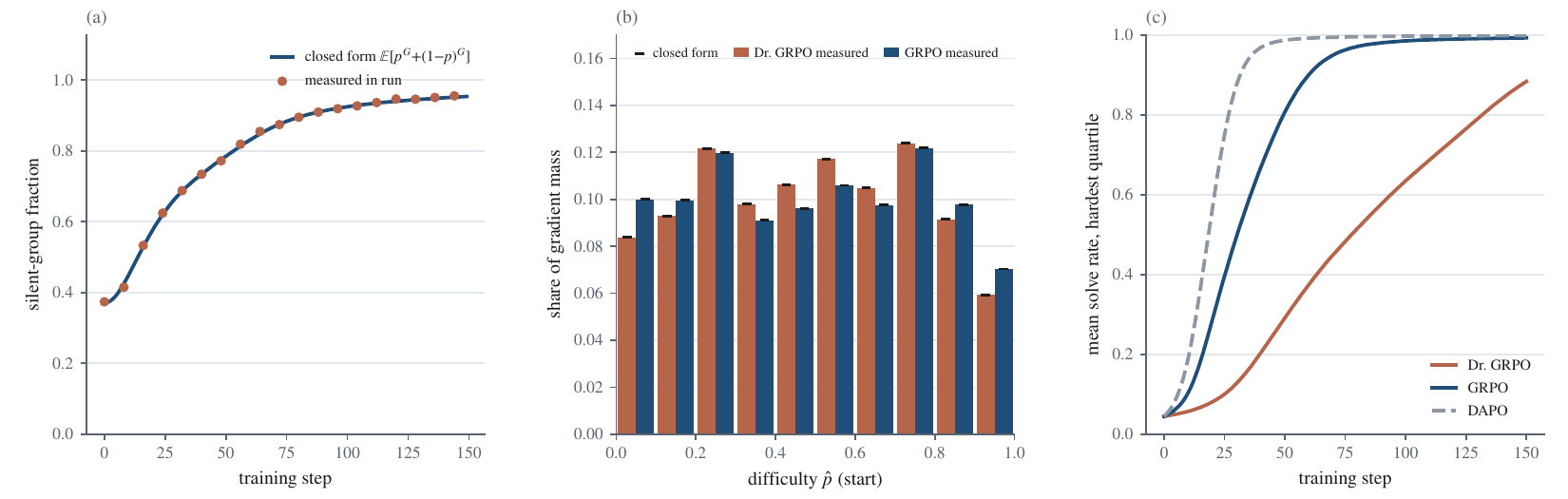}
\caption{The identity's predictions in a controlled GRPO run ($M=6{,}000$ Bernoulli-logit prompts, Big-Math initial difficulty, $G=8$). \emph{(a)} measured silent-group fraction against the closed form over training. \emph{(b)} realized gradient mass by difficulty, measured against the finite-$G$ closed form. \emph{(c)} mean solve rate of the initially-hardest quartile under GRPO, Dr.\,GRPO, and DAPO.}
\label{fig:experiment}
\end{figure}

\section{Related Work}
\label{sec:related}

\paragraph{GRPO and its critic-free relatives.} Group Relative Policy Optimization (GRPO) was introduced for mathematical reasoning \citep{shao2024deepseekmath} and popularized by R1-style training \citep{deepseekr1}. Its closest relatives drop the critic in favor of group or leave-one-out baselines: RLOO and ``back-to-basics'' REINFORCE \citep{ahmadian2024rloo,kool2019buy}, which the identity \eqref{eq:exact} subsumes through the rescaling of Proposition~\ref{prop:loo}. \citet{liu2025drgrpo} propose Dr.\,GRPO, identifying a length bias and a question-level difficulty bias in GRPO and removing both the length and standard-deviation normalizers; \S\ref{sec:bias} above expresses their difficulty bias exactly as $\partial_p\,2\arcsin\sqrt p$. DAPO \citep{yu2025dapo} adds dynamic sampling, which \S\ref{sec:degenerate} writes as removing the silent-group mass \eqref{eq:degenerate}. A parallel line keeps the group but replaces the group-relative center with a constant reference, so unanimous groups still produce a signal \citep{nie2026gradient}; a related reading recasts the same correct-minus-incorrect contrast as an implicit preference objective \citep{wu2025grpodpo}.

\paragraph{The large-group surrogate-reward view.} The asymptotics of the standard-deviation division are well understood. \citet{thrampoulidis2025advantageshaping}, in a study of Pass@K advantage shaping, establish that GRPO is (up to clipping) RLOO applied to the surrogate reward $2\arcsin\sqrt{p}$, and identify this with the binomial variance-stabilizing transform \citep{bartlett1936,anscombe1948}; related analyses read GRPO as a process reward model \citep{grpoprm2025}, give a local-curvature/adaptive-gradient account of the normalizer \citep{curvature2026}, and derive REINFORCE, PPO, and GRPO from a common expected-reward objective \citep{firstprinciples2026}. The group-standard-deviation identity \eqref{eq:exact} supplies the exact finite-group object whose group average \eqref{eq:exactE} becomes this surrogate gradient. It holds for a single group of any size $G$ in any dimension, depends on no baseline, and turns two open design choices into closed forms: the group-size law \eqref{eq:designrule} and the silent-group rate \eqref{eq:degenerate}. Those quantities are then measured on a real $215{,}608$-problem corpus in \S\ref{sec:realdata} and confirmed as training dynamics in \S\ref{sec:experiment}.

\paragraph{Variance-stabilizing transforms.} The arcsine transform $2\arcsin\sqrt p$ that stabilizes the variance of a binomial proportion is classical \citep{bartlett1936,anscombe1948} and standard in the analysis of binomial and count data. It appears here as the large-group limit \eqref{eq:exactE} of the per-group standard deviation, the object the GRPO update is shown to equal.

\paragraph{Lineage.} Methodologically this paper follows the tradition of explaining a widely used but under-theorized method by exhibiting a closed-form equivalence to a classical object, as \citet{levy2014implicit} did for word2vec \citep{mikolov2013word2vec}, and as the closed-form, search-free treatment of \citet{bay2026kore} does for scaling laws.

\section{Discussion}

The practical message is not that one existing method is universally best. The identity separates three design choices that are often discussed together. First, the objective choice: dividing by $\sigma$ gives GRPO the variance-stabilized arcsine objective, while removing the division gives Dr.\,GRPO the raw success-rate objective. Second, the compute choice: groups with $\sigma=0$ have no right-versus-wrong contrast, so DAPO-style dynamic sampling avoids spending updates on them. Third, the sampling choice: the group-size law says how large $G$ must be at a given difficulty before the finite group faithfully realizes the large-group signal.

\begin{itemize}
\item \textbf{The signal is disagreement.} By \eqref{eq:exact}, the prompt's instantaneous learning signal is computable from the sampled rollouts alone. A mixed group teaches because it contains both sides of the comparison. A unanimous group is silent because there is no within-prompt contrast.
\item \textbf{The methods are operations, not mysteries.} GRPO scales by $\sigma$, Dr.\,GRPO omits that scaling, and DAPO skips groups with $\sigma=0$. This makes the algorithmic landscape easier to reason about than a list of unrelated heuristics.
\item \textbf{Group size is a difficulty-dependent budget.} The fidelity \eqref{eq:fidelity} and silent-group rate \eqref{eq:degenerate} describe the two things $G$ buys: how much of the large-group gradient a retained group realizes, and how often a usable mixed group appears. A single uniform $G$ can be reasonable for simplicity, but the law explains why it underserves the easy and hard extremes.
\item \textbf{Standardization is a modeling choice.} Dividing by $\sigma$ trades raw-accuracy alignment for a variance-stabilized objective that gives more marginal weight to extreme difficulties. This can help rescue hard prompts, but it is also the source of the difficulty bias. The right choice depends on whether the training goal prioritizes raw success-rate alignment, hard-prompt pressure, or compute efficiency.
\end{itemize}

More broadly, reading GRPO through the group standard deviation suggests a template for analyzing other reward-shaping choices: start with one prompt, one sampled group, and one exact update. Rank-based advantages, quantile advantages, reward clipping, and length normalization should each admit the same kind of single-group accounting.

\paragraph{Limitations.} The identity \eqref{eq:exact} concerns the advantage construction under binary rewards and an on-policy first step. It deliberately sets aside clipping, the KL penalty, off-policy staleness, and non-binary rewards, each of which merits the same single-group treatment. The controlled run of \S\ref{sec:experiment} validates the closed forms as dynamics on a tractable Bernoulli-logit policy, where the score is scalar. A full language-model training loop that logs silent groups and per-difficulty gradient mass across $G$ is the natural next test; the theorem predicts what such a run should find because \eqref{eq:exact} holds in any dimension. Whether the extra weight that standardization places on the hardest prompts improves generalization beyond the trained difficulty range is a further, downstream question, continuous with the broader study of when models extrapolate past their training distribution \citep{bay2024generalization}. The DAPO comparison (\S\ref{sec:degenerate}) rests on the exact structural identification of the discarded mass; its time-evolution fit is a consistency check, not a unique prediction. A logged per-prompt solve-rate history would allow the all-correct curve to be predicted rather than fit. None of these limitations affects the core identity \eqref{eq:exact}, which is exact for any single group of any size and any policy dimension.

\section{Conclusion}

During RLVR training, a prompt teaches through the pattern of rewards assigned to its sampled answers. For binary rewards, this paper shows that one GRPO step on a prompt equals the reward standard deviation of its sampled group, $\sqrt{k(G-k)}/G$, times the contrast between correct and incorrect response scores. The result is finite-group, baseline-free, and valid in any policy dimension. A prompt pushes hardest when its answers are split and gives no signal when all sampled answers agree.

This identity makes the relationship among GRPO, Dr.\,GRPO, and DAPO concrete. GRPO divides by the group standard deviation and thereby follows the variance-stabilized arcsine objective. Dr.\,GRPO drops the division and returns to the raw success-rate objective. DAPO's dynamic sampling removes the groups where the same standard deviation is zero. The methods are therefore not three disconnected tricks; they are three operations on one number.

Because the identity holds at the finite group sizes used in training, two practical quantities become closed forms. The group-size law, $G\gtrsim1/(8\varepsilon p(1-p))$, gives the number of samples needed to realize a target fraction of the large-group gradient. The silent-group rate, $p^G+(1-p)^G$, gives the fraction of groups that provide no right-versus-wrong contrast. Both are borne out in a controlled run and on a $215{,}608$-problem difficulty corpus. The group standard deviation, long read as a normalizer, is the size of the learning signal itself.

\paragraph{Reproducibility.} \sloppy All code and data are available at \url{https://github.com/bay-yearick-lab/grpo-standard-deviation-identity}. All claims, including the general identity \eqref{eq:exact}, are verified numerically (\texttt{scripts/\allowbreak{}checks.py}), and the closed forms are exposed as a small diagnostic API (\texttt{scripts/\allowbreak{}grpo\_diagnostics.py}). The controlled run of \S\ref{sec:experiment} is \texttt{scripts/\allowbreak{}experiment\_dynamics.py}; figures and tables are regenerated from the public Big-Math solve-rate annotations by \texttt{scripts/\allowbreak{}make\_figures.py} and \texttt{scripts/\allowbreak{}analyze\_rollouts.py}. The DAPO comparison (\S\ref{sec:degenerate}) is digitized from the published figure by \texttt{scripts/\allowbreak{}digitize\_dapo.py} and stored with provenance in \texttt{data/dapo/}.

\bibliographystyle{unsrtnat}
\bibliography{references}

\end{document}

%% file: tables/generated_groupsize_table.tex
\begin{tabular}{lccc}
\toprule
 & \multicolumn{3}{c}{group size $G$ for fidelity $\varphi\ge 1-\varepsilon$} \\
\cmidrule(lr){2-4}
difficulty $p$ & $90\%$ & $95\%$ & $99\%$ \\
\midrule
$0.50$ & 7 (5) & 11 (10) & 51 (50) \\
$0.30$ / $0.70$ & 9 (6) & 14 (12) & 61 (60) \\
$0.10$ / $0.90$ & 22 (14) & 36 (28) & 144 (139) \\
$0.05$ / $0.95$ & 42 (26) & 69 (53) & 273 (263) \\
\bottomrule
\end{tabular}

%% file: tables/generated_realdata_mass_table.tex
\setlength{\tabcolsep}{5pt}
\begin{tabular}{lcc}
\toprule
gradient budget on & Dr.\,GRPO & GRPO \\
 & $p(1{-}p)$ & $\sqrt{p(1{-}p)}$ \\
\midrule
extreme ($\hat p{<}.1$ or ${>}.9$) & 13.9\% & 24.7\% \\
medium ($.4{\le}\hat p{\le}.6$) & 22.8\% & 17.5\% \\
\bottomrule
\end{tabular}

%% file: tables/generated_realdata_degenerate_table.tex
\setlength{\tabcolsep}{4pt}
\begin{tabular}{lccccc}
\toprule
group size $G$ & 4 & 8 & 16 & 32 & 64 \\
\midrule
closed form & 59\% & 44\% & 32\% & 23\% & 17\% \\
subsampled & 59\% & 43\% & 30\% & -- & -- \\
\bottomrule
\end{tabular}